\documentclass[12pt]{l4dc2021}

\title[Assured Reinforcement Learning]{Assured RL: Reinforcement Learning with Almost Sure Constraints}
\usepackage{times}
\usepackage{dsfont}

\usepackage{ifthen}
\newboolean{showcomments}
\setboolean{showcomments}{false}
\usepackage{todonotes}

\definecolor{bleudefrance}{rgb}{0.19, 0.55, 0.91}
\definecolor{ao(english)}{rgb}{0.0, 0.5, 0.0}

\newcommand{\addcite}[0]{\ifthenelse{\boolean{showcomments}}
{\textcolor{purple}{(add cite(s)) }}{}}%

\newcommand{\enrique}[1]{  \ifthenelse{\boolean{showcomments}}
{\todo[inline,color=bleudefrance]{Enrique: #1}}{}}
\newcommand{\emmargin}[1]{\ifthenelse{\boolean{showcomments}}{\marginpar{\color{bleudefrance}\tiny EM: #1}}{}}
\newcommand{\juan}[1]{  \ifthenelse{\boolean{showcomments}}
{\todo[inline,color=pink]{Juan: #1}}{}}

\newboolean{showedits}
\setboolean{showedits}{false}
\usepackage[markup=underlined]{changes}
\definechangesauthor[color=bleudefrance]{EM}
\newcommand{\aem}[1]{
\ifthenelse{\boolean{showedits}}
{\added[id=EM]{#1}}
{\!#1\hspace{-4.75pt}}
}
\newcommand{\repem}[2]{
\ifthenelse{\boolean{showedits}}
{\replaced[id=EM]{#1}{#2}}
{\!#1\hspace{-4.75pt}}
}
\newcommand{\dem}[1]{
\ifthenelse{\boolean{showedits}}
{\deleted[id=EM]{#1}}
{}
}

\author{%
 \Name{Agustin Castellano} \Email{acastellano@fing.edu.uy}
 \AND
 \Name{Juan Bazerque} \Email{jbazerque@fing.edu.uy}\\
 \addr Universidad de la República,
 Julio Herrera y Reissig 565,
 Montevideo 11400, Uruguay
 \AND
 \Name{Enrique Mallada} \Email{mallada@jhu.edu}\\
 \addr Johns Hopkins University, 
 3400 N. Charles St., Baltimore, MD 21218, USA%
}

\begin{document}

\maketitle

\begin{abstract}%
 We consider the problem of finding optimal policies for a Markov Decision Process with almost sure constraints on state transitions and action triplets. We define value and action-value functions that satisfy a barrier-based decomposition which allows for the identification of feasible policies independently of the reward process. We prove that, given a policy $\pi$, certifying whether certain state-action pairs lead to feasible trajectories under $\pi$ is equivalent to solving an auxiliary problem aimed at finding the probability of performing an unfeasible transition.
 Using this interpretation, we develop a Barrier-learning algorithm, based on Q-Learning, that identifies such unsafe state-action pairs.
 Our analysis motivates the need to enhance the Reinforcement Learning (RL) framework with an additional signal, besides rewards, called here \textit{damage function} that provides feasibility information and enables the solution of RL problems with model-free constraints. Moreover, our Barrier-learning algorithm wraps around existing RL algorithms, such as Q-Learning and SARSA, giving them the ability to solve almost-surely constrained problems.
\end{abstract}

\begin{keywords}%
  Reinforcement Learning, Constrained MDPs, Safety-critical Systems %
\end{keywords}

\enrique{Comments from meetings: Emphasize model free. Yo tengo que chocarme para aprender, pero tampoco me quiero chocar tanto... fool me once, fool me twice.}

\section{Introduction}

The last decade has witnessed a resurgence of Artificial Intelligence (AI) reaching to levels never experienced before. At the center of many of these successes, Reinforcement Learning~\citep{sutton2018reinforcement} has occupied a critical role, which when combined with modern Machine Learning techniques---such as deep neural networks~\citep{goodfellow2016deep}--- and  increased (energy-efficient) computational power~\citep{jouppi2017datacenter}, has led to astonishing demonstrations of super-human performance. While there were already instances of such accomplishments in the turn of the century~\citep{campbell2002deep,schaeffer1996chinook,schaeffer2001temporal}
\footnote{The prominent instance is the defeat of Garry Kasparov by IBM Deep Blue~\citep{campbell2002deep}.}, today's successes are pervasive and more impressive. Examples include Jeopardy!~\citep{ferrucci2012introduction}, Atari~\citep{mnih2015human}, Go~\citep{silver2016mastering}, StarCraft II~\citep{vinyals2017starcraft}, and even Poker~\citep{nichols2019machine}.


However, this success is overwhelmingly limited to virtual domains and  only ported to the physical realm after training for more than hundreds of equivalent human years~\citep{andrychowicz2020learning}.
There are several challenges that prevent the full realization of RL in physical environments, all of which are closely intertwined. Firstly, due to the high dimensionality of the spatio-temporal domain where learning occurs, thousands of trials (episodes) are needed to achieve accurate policies. Secondly, general-purpose RL algorithms lack the necessary safety guarantees that are required in such applications.
While several methods have been proposed to solve safe RL problems~\citep{Garcia}, {they often assess safety via soft penalties that lack hard guarantees.} 
This leads to the need for training using simulated environments. 
Finally, policies learned in simulated environments tend to perform poorly in practice. This further points to the need to enhance training with random permutations of the environment, which in turn makes training even more computational intensive. 

Such cyclic dependence between computing requirements, safety requirements and randomized virtual training has given rise to a renewed interest on the study of Constrained Markov Decision Processes (CMDPs)~\citep{1998.Altman}. In this setting the conventional return to be maximized is supplemented with one or several additional returns that are to satisfy certain lower bound, in expectation. Due to the additive structure of both the objective and constraints, it is possible to frame such problems as constrained optimization problem where the decision variables correspond to occupational distributions. Such approach has led to a rich body of literature that proposes RL algorithms using penalized primal methods~\citep{2006.Geibel},  primal-dual methods~\citep{dxb,2019.Paternainyfa}, and methods with additional assumptions on the model~\citep{2019.Zanon,2019.Cheng41c3}.
Unfortunately, there are some caveats with these approaches. Firstly, while expectation-based constraints of sum of rewards lead to tractable problems, such constraints are not satisfactory for safety-critical applications. Secondly, the above-mentioned solution methods either only approximately satisfy the constraints, or  guarantee constraint satisfaction asymptotically. 
{As a result, such schemes tend to experience a large number of constraint violations during training.}
 
In this work, we aim at developing Reinforcement Learning algorithms that can learn and impose safety constraints during training. Unlike \citep{2017.Achiam} and \citep{2020.Wachi}, who either require prior knowledge of the constraints or seek to learn the constraint function, our algorithm is model-free. Precisely, we consider a finite Markov Decision Process where state transitions and action triplets $(s,a,s')$ must lie within a feasibility set $\mathcal{F}$, almost surely. We show that given a policy $\pi$, the value and action value functions satisfy a barrier-based decomposition that decouples the problem of learning the feasibility set $\mathcal{F}$ from the reward process (Section \ref{sec:decomposition}). Moreover, certifying whether an initial state-action pair leads to a feasible trajectory with probability one, is equivalent to computing the action-value function of an auxiliary MDP, with identical transition probabilities, but different reward function (called here \emph{damage}), that seeks to quantify the probability of preforming an unsafe transition. {Using this equivalence, we develop a barrier-learning algorithm that learns an action-barrier function that implicitly characterizes the set of all state-action pairs $(s,a)$ that with probability one lead to a transition $(s,a,s')\in \mathcal{F}$.}  Our barrier-learning algorithm wraps around any standard RL algorithm such as Q-Learning~\addcite and SARSA~\addcite, and incrementally restricts the exploration of infeasible triplets that are learnt \emph{during training} (Section \ref{sec:assured-rl}). We illustrate our findings with numerical experiments in Section~\ref{sec:experiments} and conclude in Section~\ref{sec:conclusions}.

\section{Value function decomposition}\label{sec:decomposition}

Consider a Constrained Markov Decision Process (CMDP) with finite state space $\mathcal{S}$ and finite action space $\mathcal{A}$, a reward set $\mathcal{R}$ and a transition kernel $p$ which specifies the conditional transition probability $p(s',r\mid s,a)$, from state $s\in\mathcal{S}$ to state $s'\in\mathcal{S}$ with reward $r\in\mathcal{R}$ under action $a\in\mathcal{A}$. There are constraints that the agent must satisfy almost surely, which are specified through a set $\mathcal{F}$. A triplet $(s,a,s')$ is safe if it belongs to $\mathcal{F}$, and is unsafe otherwise. As usual, a policy $\pi$ induces a probability distribution over the action space for a given state $\pi\left(\cdot\mid s\right)$.
In this context, our goal is to maximize the value function for each possible starting state while ensuring constraint satisfaction at all times:
\begin{align}
    V^*(s):= \max_\pi \mathbb{E}_{\pi}&\left[\sum_{t=0}^\infty\gamma^t R_{t+1} ~\big|~ S_0=s\right]\label{eq:CRL1}\\
    \text{s.t.:}&\quad (S_t,A_t,S_{t+1})\in\mathcal{F}\quad\textit{a.s.}~~\forall t\nonumber
\end{align}
where the expectation in the objective is taken over the trajectories induced by $\pi$. 
Let us then define the value function $V^\pi$ for a specific policy $\pi$, in which the constraints in \eqref{eq:CRL1} are embedded inside the expectation. 
\begin{align}
    V^\pi(s) &:= \mathbb{E}_{\pi}\left[\sum_{t=0}^\infty\big(\gamma^t R_{t+1}-\mathbb{I}_{\left\{(S_t,A_t,S_{t+1})\in\mathcal{F}\right\}}\big) ~\big|~ S_0=s\right]\label{eq:v_pi}
\end{align}
where we introduced the barrier indicator function 
\begin{equation}\mathbb{I}_{\left\{\cdot\right\}}=\begin{cases}0&\text{if } \cdot \text{~is true}\\\infty&\text{if } \cdot \text{~is false}\end{cases}\label{eq:barrier_indicator}\end{equation}
The proposed value function definition will prove useful in two senses: firstly, we will show that maximizing \eqref{eq:v_pi} is the same as \eqref{eq:CRL1}. 
Secondly, the additional term in \eqref{eq:v_pi} will allow for a barrier-based decomposition of the value function, which will aid in the learning of constraints.\\
Before proceeding any further, we make a technical remark regarding the barrier term in \eqref{eq:v_pi}.

\begin{remark}[A note on the barrier term]
The reason for the addition of the indicator function term in \eqref{eq:v_pi} is to make unsafe policies yield $V^\pi(s)=-\infty$. Special care must be taken, though. Suppose there is an unsafe triplet $(s,a,s')\notin\mathcal{F}$ that occurs with zero probability under $\pi$. Unrolling the expectation in \eqref{eq:v_pi} would give a term $p\left(s'\mid s,a\right)\mathbb{I}_{\left\{\left(s,a,s'\right)\in\mathcal{F}\right\}}$, which is a zero-times-infinity indetermination. To remedy this  we instead consider the following definition for the value function
\begin{equation}
    V^\pi(s)=\lim_{\lambda\rightarrow\infty}\mathbb{E}_{\pi}\left[\sum_{t=0}^\infty\big(\gamma^t R_{t+1}-\lambda\mathds{1}_{\left\{(S_t,A_t,S_{t+1})\not\in\mathcal{F}\right\}}\big)  ~\big|~  S_0=s\right]\label{eq:v_pi_lim}
\end{equation}
where $\mathds{1}_{\{\cdot\}}$ is one if the condition is true and zero otherwise. Throughout the paper we will abuse notation and often speak of \eqref{eq:v_pi}. The reader should bear in mind that we are instead referring to \eqref{eq:v_pi_lim}.
\end{remark}
With that remark aside, we go on to show the equivalence between \eqref{eq:CRL1} and the maximization of \eqref{eq:v_pi}.

\begin{lemma}[Equivalence]
\label{lemma:equivalence}
Under the common extension where an unfeasible policy for \eqref{eq:CRL1} yields $V^\pi(s)=-\infty$, Problem \eqref{eq:CRL1} is equivalent to the maximization of \eqref{eq:v_pi}, that is

\begin{equation}
    \max_\pi V^\pi(s)=\max_\pi \mathbb{E}_{\pi}\left[\sum_{t=0}^\infty\big(\gamma^t R_{t+1}-\mathbb{I}_{\left\{(S_t,A_t,S_{t+1})\in\mathcal{F}\right\}}\big) ~\big|~ S_0=s\right]\label{eq:CRL2}
\end{equation}
\begin{proof}
If a policy $\pi_0$ is unfeasible for Problem \eqref{eq:CRL1}, then $\exists t : P\left((S_t,A_t,S_{t+1})\notin \mathcal{F}\right)>0$. This non-zero probability renders the expected value in \eqref{eq:CRL2} to $-\infty$ for $\pi_0$.
Conversely, if a policy $\pi_1$ is feasible for \eqref{eq:CRL2} then it must necessarily hold that $\left(S_t,A_t,S_{t+1}\right)\in\mathcal{F}$ almost surely $\forall t$, and hence $\pi_1$ is feasible for \eqref{eq:CRL1} as well. Therefore the feasible sets of both problems coincide. Lastly, for every feasible policy it must hold $\mathbb{I}_{\left\{(S_t,A_t,S_{t+1})\in\mathcal{F}\right\}}=0 ~\forall t$, in which case the function being maximized is the same. Then the optimal sets of the two problems coincide. 
\end{proof}
\end{lemma}
While solving \eqref{eq:CRL1} is of our utmost interest, we have just shown that, to this end, we can solve \eqref{eq:CRL2} instead. In what follows we will take this one step further, and show that \eqref{eq:v_pi} admits a \textit{barrier-based decomposition} and can be cast as the sum of two value functions: one that checks only whether the policy in consideration is feasible (which will be the main focus of this work) and one that optimizes the return, provided the policy is feasible. The main idea behind this decoupling being that the search for feasible policies will be, in practice, an easier task to undergo.\\ 
To this end we define an auxiliary \textit{barrier-based} value function $F^\pi$ that will relate to $V^\pi$.
\begin{align}
    F^\pi(s) &= \mathbb{E}_{\pi}\left[-\sum_{t=0}^\infty \mathbb{I}_{\left\{(S_t,A_t,S_{t+1})\in\mathcal{F}\right\}} ~\big|~ S_0=s\right]\label{eq:f_pi}
\end{align}
We proceed similarly for the action-value function $Q^\pi$ and its counterpart $B^\pi$
\begin{align}
    Q^\pi(s,a) &= \mathbb{E}_{\pi}\left[\sum_{t=0}^\infty\big(\gamma^t R_{t+1}-\mathbb{I}_{\left\{(S_t,A_t,S_{t+1})\in\mathcal{F}\right\}}\big) ~\big|~ S_0=s, A_0=a\right]\label{eq:q_pi}\\
    B^\pi(s,a) &= \mathbb{E}_{\pi}\left[-\sum_{t=0}^\infty\mathbb{I}_{\left\{(S_t,A_t,S_{t+1})\in\mathcal{F}\right\}} ~\big|~ S_0=s, A_0=a\right]\label{eq:b_pi}
\end{align}
Searching for policies that are optimal for \eqref{eq:CRL2} for each possible state is our original goal. By contrast, a problem such as maximizing \eqref{eq:f_pi} is one that seeks to find \textit{safe} policies, in the sense that they are \textit{feasible} for \eqref{eq:CRL2}. The main idea underpinning our work is that we can jointly work on optimizing \eqref{eq:f_pi}, which reduces the search over the policy space, while at the same time maximizing the return present in \eqref{eq:CRL2}.
In the following Theorem we establish a fundamental decomposition relationship between the value functions and their auxiliary counterparts.

\begin{theorem}[Barrier-based decomposition]
\label{thm:decomposition}
 Assume rewards $R_{t+1}$ are bounded almost surely for all $t$. Then, for every policy $\pi$ 
 \begin{equation}
     V^\pi(s) = V^\pi(s) + F^\pi(s) \label{eq:v_decomposition}
 \end{equation}
 \begin{equation}
     Q^\pi(s,a) = Q^\pi(s,a) + B^\pi(s,a) \label{eq:q_decomposition}
 \end{equation}
\begin{proof}
We shall prove the relationship in \eqref{eq:v_decomposition} regarding the value functions \eqref{eq:v_pi},\eqref{eq:f_pi}, noting that the proof for \eqref{eq:q_decomposition} is similar. The following identities hold, as explained below.
\begin{align}
    &V^\pi(s) =\mathbb{E}_{\pi}\left[\sum_{t=0}^\infty\big(\gamma^t R_{t+1}-\mathbb{I}_{\left\{(S_t,A_t,S_{t+1})\in\mathcal{F}\right\}}\big)  ~\big|~  S_0=s\right]\label{eq:proof1_1}\\
    \nonumber\\
    &= \mathbb{E}_{\pi}\left[\sum_{t=0}^\infty\big(\gamma^t R_{t+1}-\mathbb{I}_{\left\{(S_t,A_t,S_{t+1})\in\mathcal{F}\right\}}\big)  ~\big|~  S_0=s\right] + \mathbb{E}_{\pi}\left[\sum_{t=0}^\infty-\mathbb{I}_{\left\{(S_t,A_t,S_{t+1})\in\mathcal{F}\right\}} \mid S_0=s\right]\label{eq:proof1_2}
\end{align}
To show that \eqref{eq:proof1_1} can be separated as in \eqref{eq:proof1_2}, first suppose the policy considered in  \eqref{eq:proof1_1} is feasible (for Problem \eqref{eq:CRL2}). 
This necessarily implies that $(S_t,A_t,S_{t+1})\in\mathcal{F}~~\textit{a.s.}~~\forall t$, which makes the second term in \eqref{eq:proof1_2} vanish. Conversely, suppose that the policy in consideration is infeasible. This together with the fact that rewards are bounded almost surely makes \eqref{eq:proof1_1} yield $V^\pi(s)=-\infty$, which is the same value attained by both terms in \eqref{eq:proof1_2}.
\end{proof}
\end{theorem}
\enrique{I'm thinking again about our result. I think that we probably need to explain the properties of the $B^*$. This is because, it is not that we just figure out whether the optimal policy is feasible. When computing $B^*$, any policy that choses an action a s.t. $B^*(s,a)=0$ is a feasible policy. As such, our algorithm constraints the search of Q-Learning to only feasible policies. Thus Q-Learning finds the optimal policy, within the constraint set. This is a non-trivial implication of $B^*$, it somehow learns the feasible set of actions! Also, this is a property of $B^*$, not any $B^\pi$}

The preceding result implies non-trivial consequences. If the learning agent can interact with the environment and have access to rewards $R_{t+1}$ and queries of whether a transition has been safe (i.e. queries of the type $\mathbb{I}_{\left\{(S_t,A_t,S_{t+1})\in\mathcal{F}\right\}}$), then it can independently learn both $Q^\pi(s,a)$ and $B^\pi(s,a)$. Learning (and improving) $B^\pi$ is of great utility, since infeasible state-action pairs are readily coded as $-\infty$ in the corresponding function. This is further discussed in the next remark.

\begin{remark}[Properties of the optimal barrier action-value function $B^*$]
\label{rmk:B_opt}
\begin{equation}
B^*(s,a)=\max_\pi B^\pi(s,a) = \max_\pi\mathbb{E}_{\pi}\left[-\sum_{t=0}^\infty\mathbb{I}_{\left\{(S_t,A_t,S_{t+1})\in\mathcal{F}\right\}} ~\big|~ S_0=s, A_0=a\right] \label{eq:b_pi_opt}
\end{equation}
It is evident that for any policy $\pi$ the entries of $B^\pi(s,a)$ will either be $0$ or $-\infty$.
Having $B^*(s,a)=-\infty$ means that if starting at $s$ with action $a$ and then following any policy, there is an (albeit small) non-zero probability that an unsafe event will be encountered.\\ Conversely, if $B^*(s,a)=0$ then following the optimal policy guarantees we will always ensure constraints while starting from state $s$ and action $a$.
Furthermore, if $B^*$ is available, then any policy that chooses an action such that $B^*(s,a)=0$ is a feasible policy for that state. If, on the other hand, the policy chooses an action which leads to $B^*(s,a)=-\infty$, then that policy can be deemed unsafe immediately, and can be discarded. Notice, however, that this is a particular property of the optimal function $B^*$: for a sub-optimal $\pi$, $B^\pi$ can be padded with lots of $-\infty$ in places where $B^*$ has zeros. Intuitively, this would mean that the sub-optimal policy starting from $(s,a)$ makes unsafe transitions along its trajectory.
In this sense, having access to the optimal action-value function $B^*$ helps constraint the search of any other known algorithms (such as Q-Learning) to only feasible policies. This becomes a joint work of optimizing the return (using one of many possible techniques), while learning the feasible set of actions at the same time. 
\end{remark}


\section{Almost sure constraints}
In this section we model the probability of  satisfying the constraints, laying the ground to construct the algorithm that guarantees safety almost surely. In our way there we connect to the state of the art on safe learning with cumulative expected constraints.

In this direction, and given the triplet $(S_t,A_t,S_{t+1})$, consider the auxiliary random variable $D_{t+1}\in\{0,1\}$ which indicates if $(S_t,A_t,S_{t+1})\in\mathcal F$, that is, $D_{t+1}=1$  if  $(S_t,A_t,S_{t+1})\in\mathcal F$ and zero otherwise. Intuitively, $D_{t+1}$ indicates if there is \emph{damage} in the transition from $S_t$ to $S_{t+1}$, and it is related to the indicator function \eqref{eq:barrier_indicator} through $D_{t+1}=\exp\left(-\mathbb I_{\left\{(S_t,A_t,S_{t+1}\in \mathcal F)\right\}}\right)$.

Accordingly, the conditional transition probability model is augmented to incorporate this new variable, so that $p(s',d|s, a)$ represents the probability of  evolving from state $s$ to state $s'$ facing damage $d$, under action $a$. We further make the following assumptions on  $p(s',d|s, a)$ {that will allow us to learn in episodes. We assume that the trajectory ends when there is damage, which is modeled by moving to an absorbent terminal state $s_D$. This state is reached the first time there is damage $D_{t+1}=1$, yielding  $D_t=0$ thereafter.} Specifically, we set the following conditions on the transition probabilities
\begin{align}
&P(S_{t+1}=s_D,D_{t+1}=0| S_t \neq s_D,A_t=a_t)=0\label{eq:cond_1}\\
&P(S_{t+1}\neq s_D,D_{t+1}=1| S_t \neq s_D,A_t=a_t)=0\label{eq:cond_2}\\
&p(S_{t+1}=s_D,D_{t+1}=0| S_t = s_D,A_t=a_t)=1\label{eq:cond_3}
\end{align}
The first and second conditions state that the transition to the terminal state $s_D$ happens if and only if there is damage $D_{t+1}=1$, while the  third  condition ensures that $s_D$ is absorbent with $D_{t+1}=0$ thereafter. Next, we consider  the probability of violating the constraints at some point through an entire trajectory starting from $S_0=s$, with initial action $A_0=a$ and selecting subsequent actions according to a particular policy $\pi$.
\begin{align}
    q^\pi_D(s,a)&:=P\left(\bigcup_{t\geq 0} \{(S_{t},A_t,S_{t+1})\notin \mathcal F\}\Big|S_0=s,A_0=a\right)=P\left(\bigcup_{t\geq 0}\{ D_{t+1}=1\}\Big| S_0=s,A_0=a\right) \label{eq:pede}
    \end{align}
We also define the expected cumulative damage starting from $S_0=s,\ A_0=a$ and following policy $\pi$ as
\begin{align}\label{eq:depi}
d^\pi(s):=    \mathbb{E}_\pi\left[\sum_{t=0}^\infty D_{t+1}=1|S_0=s,A_0=a\right]
\end{align}
It can be shown that \eqref{eq:pede} and \eqref{eq:depi} coincide~\citep{2019.Paternainyfa}. Indeed, since  conditions \eqref{eq:cond_1}-\eqref{eq:cond_3} ensure that damage $D_{t+1}$ equals one at most once in the entire trajectory, the union in \eqref{eq:pede} is disjoint, and we can write 
\begin{align}
    q_D^\pi(s,a)&=\sum_{t=0}^\infty P\left(D_{t+1}=1|S_0=s,A_0=a\right)\label{eq:pd1}\\
&=\sum_{t=0}^\infty \mathbb{E}_\pi\left[D_{t+1}=1|S_0=s,A_0=a\right]=d^\pi(s,a)\label{eq:pdsum}
\end{align}
The next step is to link  $q_D(s,a)$ and $d^\pi(s,a)$ with the barrier function $B^\pi(s,a)$ in \eqref{eq:b_pi}.
In the next statement we consider a policy $\pi$ which  satisfies all constraints almost surely when starting from $S_0=s$ and $A_0=a$, so that  $q_D^\pi(s,a)=0$ and prove that for such a safe policy $B^\pi(s,a)=0$.  

\begin{theorem}
 Under conditions \eqref{eq:cond_1}-\eqref{eq:cond_3} \ $q_D^\pi(s,a)=0$ iff $B^\pi(s,a)=0$.  
\begin{proof} Starting from \eqref{eq:pd1}, and with the probabilities being non negative, we can write
\begin{align*}
q_D^\pi(s,a)=0\ &\Longleftrightarrow\  P(D_{t+1}=1|S_0=s,A_0=a)=0,\  \forall t\geq 0\\
    &\Longleftrightarrow\  P((S_{t},A_t,S_{t+1})\in\mathcal F |S_0=s,A_0=a)=0,\ \forall t\geq 0\\
    &\Longleftrightarrow\  \mathbb{E}_\pi\left[\mathbb I_{\left\{(S_{t},A_t,S_{t+1})\in\mathcal F\right\}}|S_0=s,A_0=a\right]=0,\ \forall t\geq 0\ \Longleftrightarrow\  B^\pi(s,a)=0
\end{align*}
 The second equivalence follows from the definition of $D_{t+1}$, the third one from the definition of $\mathbb I_{\{\cdot\}}$ with the  limit in \eqref{eq:v_pi_lim},  and the fourth one  from linearity and $\mathbb I_{\{\cdot\}}$ being non-positive.  

\end{proof}
\end{theorem}
Finally we take advantage of the identity between $d^\pi(s,a)$ and $q_D^\pi(s,a)$ to  derive a Bellman equation for the probability of being safe. 

\newpage
\begin{theorem}[Bellman equation for $q_D^\pi$]
\label{thm:Bellman}
 Under conditions  \eqref{eq:cond_1}-\eqref{eq:cond_3} the conditional probability of violating the constraints satisfies 
 $     q_D^\pi(s,a)=\mathbb{E}\left[D_{t+1}+ q_D^{\pi}(S_{t+1},A_{t+1}) |S_t=s,A_t=a\right]$.  
\end{theorem}
\begin{proof} We can set the first term of \eqref{eq:pdsum} apart and identify  the  conditional expectation of $D_{t+1}$ in it
\begin{align}
     q_D^\pi(s,a)&=\sum_{\tau=t}^\infty P(D_{\tau+1}=1|S_t=s,A_t=a)\\&=P(D_{t+1}=1|S_t=s,A_t=a)+\sum_{\tau=t+1}^\infty P(D_{\tau+1}=1|S_t=s,A_t=a)\label{eq:qdepe}
\end{align}
 By setting $Z_t=(S_t,A_t)$ to reduce notation, we  write the second term in \label{eq:qdepe} as
 \begin{align}
     &\sum_{\tau=t+1}^\infty P(D_{\tau+1}=1|S_t=s,A_t=a)=\sum_{\tau=t+1}^\infty\sum_{s'}\sum_{a'} P\left(D_{\tau+1}=1,Z_{t+1}=(s',a')|Z_t=(s,a)\right)\nonumber\\
     &=\sum_{\tau=t+1}^\infty\sum_{s'}\sum_{a'} P(D_{\tau+1}=1|Z_{t+1}=(s',a'),Z_t=(s,a))P(Z_{t+1}=(s',a')|Z_t=(s,a))\label{eq:Zt1Zt}
     \\&=\sum_{\tau=t+1}^\infty\sum_{s'}\sum_{a'} P(D_{\tau+1}=1|Z_{t+1}=(s',a'))P(Z_{t+1}=(s',a')|Z_t=(s,a))\label{eq:Zt1}\\
         &=\sum_{s'} \sum_{a'}q_D^{\pi}(s',a') P(Z_{t+1}=(s',a')|Z_t=(s,a))=E[ q_D^{\pi}(Z_{t+1}) |Z_t=(s,a)]
 \end{align}
 where we used the Markov property of the transition probability for  dropping $Z_t$ in \eqref{eq:Zt1Zt}. 
\end{proof}

The Bellman recursion in Theorem \ref{thm:Bellman} lets us compute the probability of being unsafe recursively, as the expected damage after the first transition, plus the average probability of damage starting from the next state.  As we  want to know if $q^\pi_D(s,a)=0$, we can stop if there is  damage after the first transition, and otherwise  step to the next state and repeat. This is the idea behind the barrier-learning Algorithm in next section .


\section{Assured reinforcement learning }\label{sec:assured-rl}

As introduced in previous sections, our goal is to learn a safe policy $\pi$ that maximizes the cumulative reward and finds trajectories that satisfy the constraints at all time steps almost surely. This safety condition on the policy can be described by the probability $q_D^\pi(s,a)$ being null for all $s$ in the set of starting points $\mathcal S_0$, or equivalently $B^\pi(s,a)=0$ for all $s \in \mathcal{S}_0,\ a\in \mathcal A$. As before, we assume that  each state transition yields a reward $R_{t+1}$ and a damage index $D_{t+1}$, which are distributed according to the transition probability $p(s'd,r|s,a)$. {Algorithm \ref{alg:barrier} specifies how to use this data in order to learn the set of safe policies. The Barrier  function $B(s,a)$ is updated incorporating the new information about safety provided by the damage index $D_{t+1}$ at each transition}.
\begin{algorithm}[htbp]
\SetAlgoLined
\KwData{$B$-function (initialized as all-zeroes)}
\KwIn{$(s,a,s',d)$ tuple}
\KwOut{Evaluated barrier-function $B(s,a)$}
$B(s,a)\leftarrow B(s,a)+\log(1-d)+\max_{a'}B(s',a')$

\KwRet{$B(s,a)$}
\caption{Barrier-learner}
\label{alg:barrier}
\end{algorithm}\\
Then, Algorithm \ref{alg:assured_ql} uses the decomposition of Theorem \ref{thm:decomposition} to embed this safety information in the q-function $Q(s,a)$.  Hence, we identify that the condition  $Q(s,a)=-\infty$ is equivalent to $B(s,a)=-\infty$ and this propagates information about safety from successor states. Indeed,  $Q(s,a)=-\infty$ flags the pair $(s,a)$ as one leading, with nonzero probability, to imminent damage or to a next state where all further actions lead to damage. Such a pair $(s,a)$ does not comply with our requirement  of satisfying the constraints almost surely along the trajectory, and thus it is discarded from the  set of state-action pairs that ensure feasibility. This is obtained by defining safe sets $\mathcal{B}(s)$ which specify the permitted actions at state $s$. If an action is deemed \emph{unsafe} (by inspecting $Q(s,a)=-\infty$) then it is taken out from $\mathcal{B}(s)$. In this sense, at every time step we are enforcing that the set $\mathcal{B}(s)=\text{argmax}_{a} B(s,a)$.
\SetKw{Break}{break}
\SetKw{Continue}{continue}
\begin{algorithm}[htbp]
\SetAlgoLined
\DontPrintSemicolon
\KwData{Starting state distribution $\mu$, discount $\gamma$, exploration $\epsilon$, learning rate $\eta$}
\KwResult{Optimal action-value functions $Q^*$ and $B^*$}
Initialize safe sets $\mathcal{B}(s)=\mathcal{A} \quad\forall s$\;

$Q(s,a)=B(s,a)=0 \quad\forall ~(s,a)$ \;

\For{$episode=0,1,\ldots$}{
Draw $s_0 \sim \mu$

Set $s \leftarrow s_0$ \;

\If{$\mathcal{B}(s)=\emptyset$}{\Continue}

$k \leftarrow 0$ \;

\While{$s \neq s_{D}$}{
Draw action $a\sim\pi$ \quad\big(e.g. $\epsilon$-greedily w.r.t. $Q(s,\cdot)$, with $a\in\mathcal{B}(s)$\big)\;

Apply $a$, observe $(s', r, d)$\;

$B(s,a) \leftarrow \texttt{barrier\_learner(}s,a,s',d\texttt{)}$\;

$Q(s,a) \leftarrow B(s,a) + (1-\eta)Q(s,a) + \eta\big(r+\gamma\max_{a'}(Q(s',a')\big)$ \;

\If{$Q(s,a)=-\infty$}{
$\mathcal{B}(s)\leftarrow \mathcal{B}(s)~\backslash~\{a\}$\;
}
$k\leftarrow k+1$

$s\leftarrow s'$

\If{$\mathcal{B}(s)=\emptyset$}{\Break}

}
}
\caption{Episodic Assured Q-Learning}
\label{alg:assured_ql}
\end{algorithm}

This in turn  narrows the set of feasible policies that are contenders for being optimal in terms of the reward. Notice that $Q(s,a)$  also incorporates the reward, and coincides with the standard q-function when the barrier is null, so that our Assured Q-Learning algorithm learns to optimize the reward  while  assuring feasibility on the fly.

\juan{can we declare that it converges or that if it converges then all states are safe? Then we could say that among the state actions that remain the effect of B is zero and Q is optimal as in the standard case? Not sure about these claims.}
\enrique{Que el Barrier Algorithm converge a.s. sale the el paper de Tsitsiklis de Q-Learning y nuestros resultados de la seccion 3. Creo que la del Q tambien sale en forma similar. La gran diferencia para mi que vale la pena investigar a futuro es que para mi el Barrier Learning Alg. converge mucho mas rapido. }

\begin{remark}[Avoiding unsafe states] Keeping track of the sets of safe actions $\mathcal B(s)$   in Algorithm \ref{alg:assured_ql} implies that we can identify unsafe states as those with $\mathcal B(s)=\emptyset$. This gives us a second level of safety during training, as we  prematurely stop an episode if we reach one of  such unsafe states. In this way, we prevent what we already identified as unavoidable future damage, and assume $d=1$ for updating $B(s,a)$ and $Q(s,a)$ without actually following the trajectory until the actual damage occurs. Notice that this propagates backwards,  so that we can identify unsafe states many steps ahead of where actual damage would occur.
\end{remark}

\section{Numerical Experiments}\label{sec:experiments}
In this section we compare the performance of Algorithm \ref{alg:assured_ql} against standard Q-Learning. We start by introducing the CMDP: a simple grid-world or maze that the agent must learn to navigate. We finish by comparing the learning curve of both algorithms and the number of times constraints have been violated during training.


\subsection{Experimental setup}

\begin{figure}[htbp]
    \centering
    \includegraphics[width=.5\linewidth]{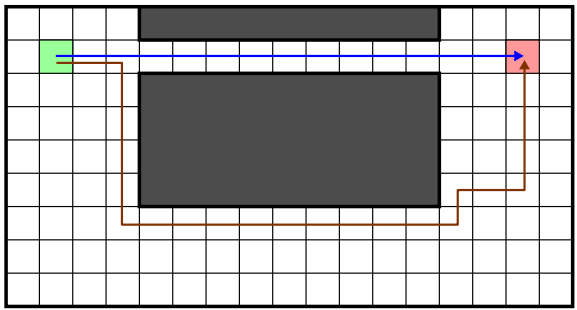}
    \caption{Example grid-world. The agent starts at the green square and if he reaches the red ``goal" achieves a reward of $+10$. Moving around has null reward and bumping into walls is unsafe. With discounting ($\gamma<1$), the optimal trajectory is the one that goes through the corridor (shown in blue). During learning, the corridor is potentially unsafe since many actions lead to a wall. As a result, some agents converge to sub-optimal trajectories that avoid the corridor (as the one shown in brown).}
    \label{fig:gridworld}
\end{figure}

We test our Algorithm on a simple grid-world shown in Figure \ref{fig:gridworld}. The agent must learn to navigate the maze, starting from the green ``start'' state and finishing in the red ``goal'' state. Available actions are \textit{up, down, right}~and~\textit{left}, with which the agent transitions to the corresponding neighbouring state, if possible. If the agent bumps into a wall ---which is considered \textit{unsafe}--- the episode terminates and the new episode begins in the preceding state.  Rewards are $0$ for moving around and $+10$ if the goal is reached. \\
We deploy Algorithm \ref{alg:assured_ql} and see how it compares against standard Q-Learning. To provide a fair comparison, when doing standard Q-Learning bumping into walls will have a reward $r=-\infty$. In any case, by specifying a discount factor $\gamma<1$ the optimal policy will be one that learns to reach the goal in the least number of steps. This policy is the one that goes through the narrow corridor, which is tricky to learn since 2 out of the 4 actions available lead to a wall inside the corridor. Suboptimal policies often learn to avoid the corridor and take the long way around (see Figure \ref{fig:gridworld}).

\subsection{Results}
We define the \textit{total episode length} as the number of steps it takes the agent to reach the goal while starting from the ``start'', and consider this to be a proxy for learning. If after $100$ steps the agent hasn't reached the goal, the episode is reset and the agent is taken to its starting position. 
We train $1000$ instances of each agent on the gridworld described, setting $\gamma=0.9$, $~\epsilon=0.1$,$~\eta=0.1$. Figure \ref{fig:c3_results} on the left shows the total episode length during training. As can be seen, the learning curve is similar for both Algorithms. The interesting thing to check is the number of \textit{constraint violations} during training, which is shown by Figure \ref{fig:c3_results} on the right. It depicts the cumulative constraint violations while learning ---that is, the number of times the agent has bumped into a wall. Notice that after some time the Assured agent's constraint violations plateau, while the Standard Q-Learning agent keeps taking unsafe actions. It is worth recalling that our Assured agent discards actions when he learns they lead to unsafe transitions, whereas the Q-Learning agent might still take an unsafe action (even if it has a corresponding value $Q(s,a)=-\infty$) as a consequence of the $\epsilon$-induced exploration.





\begin{figure}[htbp]
    \centering
    \includegraphics[width=\linewidth]{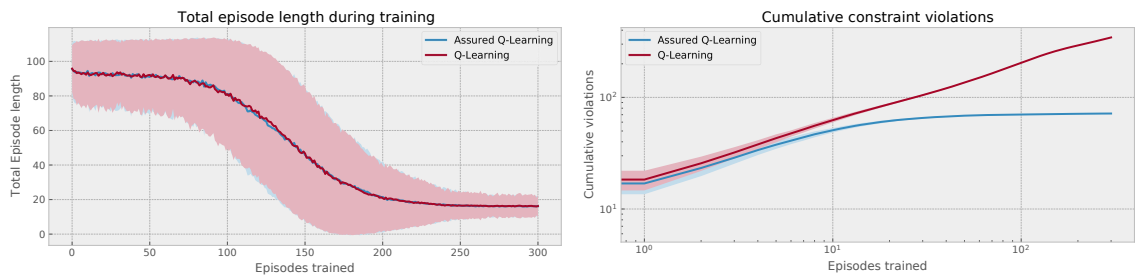}
    \caption{On the left: episode length during training, i.e., how many steps the agent takes to reach the goal. On the right: log plot of the cumulative constraint violations during training, i.e., the number of times the agent has bumped into a wall. Solid marks correspond to the average value over $N=1000$ independent trials, while shaded regions are the average value $\pm \sigma/\sqrt{N}$, with $\sigma$ the sample deviation. Notice that the learning curves (left) for both algorithms are similar. However, the Assured agents learn to avoid walls, why the other agents do not (right).}
    \label{fig:c3_results}
\end{figure}



\section{Conclusions}\label{sec:conclusions}
We addressed the problem of safe learning on a model-free framework, where optimal policies and constraints are hinted by the realization of rewards and damages. In this context, we defined value functions that can be decomposed in a constraint-abiding term and a term that optimizes the return. This barrier function can be inferred from damage data, and is crucial to learn the set of safe policies. Our proposed barrier-learning algorithm can be combined with Q-Learning, and learn to optimize rewards and discard unsafe policies simultaneously on the fly.

\clearpage 

\bibliography{refs}

\end{document}